\documentclass[12pt]{amsart}
\usepackage{amsmath, amssymb, amsthm}  % AMS packages for enhanced math features
\usepackage{hyperref}  % For hyperlinks
\usepackage{graphicx}  % For including graphics (optional)
\usepackage{cite}      % For better handling of citations
\usepackage{geometry}  % For page layout and margins (optional)
\usepackage{enumerate} % For custom list styles (optional)
\usepackage[all]{xy} 

\theoremstyle{plain}  % Use for default style (italicized text)
\newtheorem{theorem}{Theorem}[section]
\newtheorem{lemma}[theorem]{Lemma}
\newtheorem{proposition}[theorem]{Proposition}
\theoremstyle{definition}  % Use for definitions (non-italicized text)
\newtheorem{definition}[theorem]{Definition}
\newtheorem{assumption}[theorem]{Assumption}  % Define 'assumption' environment
\newtheorem*{remark}{Remark}
\title[Markov Chain Gradient Descent in Hilbert Spaces]{Markov Chain Gradient Descent in Hilbert Spaces
}
\author{Priyanka Roy}
\address{Institute for Mathematical Methods in Medicine and Data-Based Modeling\\
         Johannes Kepler University Linz, Altenberger Strasse 69, A-4020 Linz, Austria}
\email{priyanka.roy@jku.at}
\author{Susanne Saminger-Platz}
\address{Institute for Mathematical Methods in Medicine and Data-Based Modeling\\
         Johannes Kepler University Linz, Altenberger Strasse 69, A-4020 Linz, Austria}
\email{susanne.saminger-platz@jku.at}

\keywords{Markov chain, Online regularized learning, Stochastic approximation, Reproducing kernel Hilbert spaces}
\subjclass[2020]{60J20, 68T05, 68Q32, 62L20}

\date{\today}

\def\tsc#1{\csdef{#1}{\textsc{\lowercase{#1}}\xspace}}
\tsc{WGM}
\tsc{QE}
\tsc{EP}
\tsc{PMS}
\tsc{BEC}
\tsc{DE}
\newcommand{\opT}[1]{{T}_{k,#1}}
\newcommand{\opI}{{I}}

\newcommand{\N}{\mathbb{N}}

\begin{document}

\begin{abstract}
In this paper, we study a Markov chain-based stochastic gradient algorithm in general Hilbert spaces, aiming at approximating the optimal solution of a quadratic loss function. We establish probabilistic upper bounds on its convergence. We further extend these results to an online regularized learning algorithm in reproducing kernel Hilbert spaces, where the samples are drawn along a Markov chain trajectory.
%\PACS{PACS code1 \and PACS code2 \and more}
%\subclass{60J20 \and 68T05 \and 68Q32 \and 62L20}
\end{abstract}
\maketitle
\section{Motivation}
Studying learning algorithms in the context of non-i.i.d.~settings such as those involving Markov chains has received quite some attention, as many real-world scenarios violate the i.i.d.~assumption (see, e.g., \cite{besenczi2021large, huang2020Markov}). Motivated by this, we focus on a specific class of stochastic gradient algorithms and aim to address the following key questions: (i) How can one formulate and analyze a Markov chain gradient descent algorithm in Hilbert spaces specifically when the initial state of the chain is not drawn from its stationary distribution? (ii) What error bounds and learning rates can be established for the Markov chain gradient algorithm in this setting? (iii) How can these results be applied to a specific regularized learning algorithm, namely the online regularized learning algorithm in reproducing kernel Hilbert spaces (RKHS)? The aim of this paper is to systematically address these three questions in the order presented.
\section{Introduction}
 A growing body of work has analyzed the impact of Markov chain sampling, specifically it's mixing time on the convergence rates and generalization guarantees of gradient descent algorithms across various settings for e.g., \cite{MR3023783} analyzed mirror descent with samples generated by an ergodic Markov chain. Using a step size proportional to $1/\sqrt{T t_{\text{mix}}}$, their three main convergence theorems show that Ergodic Mirror Descent converges at a rate of $\mathcal{O}\left( \sqrt{\frac{t_{\text{mix}}}{T}} \right)$ for a broad class of ergodic processes, both in expectation and with high probability, even when the mixing time $t_{\text{mix}}$ is random. Their matching lower bounds confirmed that their rate is tight up to constants. Furthermore, \cite{sun2018Markov} extends the analysis to non-reversible finite-state Markov chains and nonconvex objectives, allowing the mixing time of the underlying Markov chain to vary across different mixing levels, rather than being fixed at a single level. \cite{doan2020finitetimeanalysisstochasticgradient} complement these results by removing bounded-gradient and bounded-iterate assumptions, demonstrating that gradient descent based on an ergodic chain converges at the same rate as in the i.i.d.\ case, up to an extra logarithmic factor that accounts for the mixing time of the Markov chain. In the context of linear regression, \cite{nagaraj2020least} show tight information‐theoretic lower bounds on the estimation error that grow with the chain’s mixing time,~\(t_{\mathrm{mix}}\),  under different noise settings, and further study two methods: Stochastic gradient descent--Data Drop (SGD-DD), which subsamples every \(\widetilde\Theta(t_{\text{mix}})\) samples to recover optimal rates, and experience replay, which buffers past transitions to reduce the mixing penalty to \(\sqrt{t_{\text{mix}}}\). \cite{doan2020convergence} later considers an accelerated ergodic Markov chain SGD for both convex and non-convex problems and establish that for decaying step sizes the rates in the Markovian case incur an additional logarithmic factor of the mixing time campared to the i.i.d. case. More recently, \cite{dorfman2022adapting} devise a first-order method that adapts to the chain’s unknown mixing time and still achieves the optimal \(\widetilde{\mathcal{O}}\bigl(\sqrt{t_{\text{mix}}/T}\bigr)\) rate in the convex case. Parallel lines of work focus on generalization: \cite{wang2022stability} prove that Markov chain gradient descent attains the same stability and \(\mathcal{O}(n^{-1/2})\) excess-risk guarantees as with i.i.d~samples. Finally, \cite{even2023stochastic} analyzes stochastic gradient descent with Markovian samples, showing convergence under much weaker assumptions. He first provides lower bounds, indicating convergence rates are constrained by the chain’s hitting time. Then, under significantly milder conditions (without bounded gradients or iterates), he derives upper-bound rates for Markov chain gradient descent tied to mixing time.
\par
Finally, in the context of the online regularized learning algorithm, we briefly highlight early work in this direction. In particular, \cite{smale2009online} investigate an online regularized  learning algorithm with a varying regularization parameter, where training samples are generated by a Markov chain. Exploiting the chain’s exponentially fast mixing i.e., it's assumed that the marginal distributions on the input space \(X\) converges to stationarity at an exponential rate for which they prove almost sure convergence of the iterates to the true regression function. They further show that if each sample distribution lies within $\epsilon$ of the target distribution, the algorithm still converges with an additional factor that grows linearly with $\epsilon$. Finally, they show that waiting until the chain is $\epsilon$-close to stationarity before recording a sample decorrelates successive observations but discards all the intermediate states. However, using every step keeps those correlated samples, and by carefully controlling the accumulated drift through the chain’s exponential mixing, they obtain almost-sure convergence that is faster than the algorithm using mixing-time samples only. Note that their definition of exponential convergence of the marginal distributions to the stationary distribution is stronger, as it holds uniformly over all initial states of the Markov chain i.e., the chain is uniformly ergodic.  Next, \cite{zhang2023online} analyze a simple, regularized stochastic gradient update in an RKHS for nonparametric regression without assuming i.i.d. data. They prove that the estimator converges in mean square under two broad non‐i.i.d. settings, i.e., for weakly dependent, non-stationary streams whose instantaneous covariance operators remain persistently excited and for independent samples drawn from drifting probability measures.
\par
In contrast to previous analyses as presented above, our work specifically investigates the performance of a Markov chain gradient descent algorithms in general Hilbert spaces by deriving convergence rates. Furthermore, we are motivated to study this particular setting as it allows us to study the performance of online regularized learning algorithm based on Markov chain samples due to their connection to Markov chain gradient descent algorithm. In what follows, we provide a brief overview of our problem setup along with the advances established in our results.
\par
 Consider a real Hilbert space \( W \) and a time-homogeneous Markov chain \( (z_t)_{t \in \mathbb{N}} \) on the measurable space \( (Z, \mathcal{B}(Z)) \), with transition kernel \( P \) and a unique stationary distribution \( \rho \). Given a quadratic loss function \( V : W \times Z \to \mathbb{R} \), let \( \nabla V_z(w) \) denote its gradient with respect to the first argument. Starting from an initial point \( w_1 \in W \), we define the iterative update sequence via stochastic gradient descent
\begin{equation}\label{genupdt}
     w_{t+1} = w_t - \gamma_t \nabla V_{z_t}(w_t), \qquad t \in \mathbb{N},
\end{equation}
where \( (\gamma_t) \) is a positive step-size sequence chosen as \( \gamma_t = t^{-\theta} \) for some \( \theta \in \left(\frac{1}{2},1\right] \).

Our goal is to investigate the convergence behavior of the sequence \( (w_t)_{t \in \mathbb{N}} \) toward the unique minimizer \( w^\star \in W \) of the expected loss i.e,  \(\mathbb{E}_{z\sim\rho} \left[\nabla V_z(w^\star)\right] = 0 \). Since the Markov chain may start from a non-stationary distribution, the mixing time \( t_{\mathrm{mix}} \) of the chain plays a role in the convergence analysis. Specifically, we show that for any \( \theta \in \left( \tfrac{1}{2}, 1 \right) \), the error satisfies
\[
    \|w_t - w^\star\| = \mathcal{O}\!\left( t^{-\theta/2} \sqrt{t_{\mathrm{mix}}} \right),
\]
and in the boundary case \( \theta = 1 \), the rate becomes
\[
    \|w_t - w^\star\| = \mathcal{O}\!\left( t^{-\alpha/2} \sqrt{t_{\mathrm{mix}}} \right),
\]
where \( \alpha \in \left(0,\frac{1}{2}\right] \). Note that for the cases  when the Markov chain mixes rapidly, the convergence bounds closely match those obtained in the i.i.d.\ setting analyzed by \cite{MR2228737} and are sharper in this sense. Studying Markov chain gradient descent in Hilbert spaces offers some advantages, which we highlight in the following discussion. As a concrete example, consider the case \( W = \mathcal{H}_K \), a reproducing kernel Hilbert space (RKHS). For a fixed \( z = (x,y) \in Z \), we define the quadratic potential loss by
\[
    V_z(f) = \frac12 \big( (f(x) - y)^2 + \lambda \|f\|_K^2 \big).
\]
This loss fits into the framework of Eq.~\eqref{genupdt} with \( W = \mathcal{H}_K \), leading to the identification \( w_{t+1} = f_{t+1} \) and \( w^\star = f_{\lambda, \mu} \). The corresponding update rule becomes
\begin{equation}\label{onlinealgo}
    f_{t+1} = f_t - \gamma_t\big( (f_t(x_t) - y_t) K_{x_t} + \lambda f_t \big), \quad f_1 \in \mathcal{H}_K \ \text{(e.g., } f_1 = 0\text{)},
\end{equation}
which is precisely the online regularized learning algorithm. We shall later elaborate on the characterization of \( f_{\lambda, \mu} \). This connection shows that the convergence results for Markov chain gradient descent in general Hilbert spaces help in yielding corresponding guarantees for online regularized algorithms in RKHSs driven by Markov samples, which constitutes the second main contribution of this work.
\par 
To outline our paper, we begin in Section~\ref{prelim1} with a brief overview of the mixing time of a Markov chain, followed by Section~\ref{sec:mcsg} which introduces a Markov chain gradient descent algorithm in Hilbert spaces and establish our main convergence results. These results are subsequently extended to obtain convergence rates for an online regularized learning algorithm with Markov samples in Section~\ref{application}. Finally, \ref{appen} compiles some established inequalities that are essential for our analysis.
\section{Preliminaries on Markov chains}\label{prelim1}
Let $Z$ be a measurable space equipped with its Borel $\sigma$‑algebra $\mathcal{B}(Z)$, and consider a time‑homogeneous Markov chain $(z_t)_{t \in \mathbb{N}}$ on $Z$ whose one‑step transition kernel is denoted by
\[
P(z, A) = P(z_{t+1} \in A \mid z_t = z), \quad z \in Z,\; A \in \mathcal{B}(Z).
\]
Note that we write $P^t (z, A)$ to denote the probability of reaching $A$ in $t$ steps, had the chain started at $z$ i.e., the $t$‑step transition probability given as
\[
P^t (z,A) = P(z_{t+1} \in A \mid z_1 = z).
\]
\begin{definition} Let $(z_t)_{t \in \mathbb{N}}$ be a time-homogeneous Markov chain with stationary distribution $\rho$ and let $\epsilon \in \mathbb{R}_{>0}$. The mixing time $t_{\text{mix}}(\epsilon)$ of the Markov chain is defined as 
\begin{equation}
t_{\text{mix}}(\epsilon)=\min \{ t \geq 1 : d(t) \leq \epsilon \}, 
\end{equation}
where
\begin{equation}
d(t) = \sup_{z \in Z} \| P^t(z,\cdot) - \rho(\cdot) \|_{\text{TV}}.
\end{equation}
For any two probability measures $\mu$ and $\nu$ on $(Z, \mathcal{B}(Z))$, their total‑variation distance is given by
\[
\|\mu - \nu\|_{\mathrm{TV}} = \sup_{A \in \mathcal{B}(Z)} |\mu(A) - \nu(A)|.
\]
\end{definition}
Note that in the rest of what follows, we shall denote $t_\text{mix}:=t_{\text{mix}}\left(\tfrac{1}{4}\right)$  the mixing time of the Markov chain.
We next provide an example of a specific Markov chain together with its associated mixing time.
\subsection{Example}
A configuration model is a family of random graph models in which the degrees of vertices are fixed beforehand while randomizing all other structure. In this context, \cite{MR3843821} introduce a \textit{discrete-time dynamic} analogue of the configuration model, in which, at each time step, a fraction $\alpha_n$ of edges is selected and rewired uniformly at random crucially preserving each vertex’s degree. Under mild regularity conditions on the prescribed degree sequence, and assuming the dynamics are sufficiently fast so that
\(
\lim_{n \to \infty} \alpha_n (\log n)^2 = \infty,
\)
they prove that the non-backtracking random walk mixes in significantly less time than the $\log n$ order known for the static model. Specifically, for any fixed $\varepsilon \in (0,1)$, the $\varepsilon$-mixing time for random walk \textit{without backtracking} grows like
\(
\frac{\sqrt{2 \log(1/\varepsilon)}}{\log(1 / (1 - \alpha_n))} \quad \text{as } n \to \infty,
\)
with high probability. For additional examples of mixing time for various other Markov chains, see for e.g., \cite{vigoda1999improved,MR2506768}.    
\section{Markov chain based gradient descent algorithm in Hilbert spaces}\label{sec:mcsg}
  We shall consider a Markov chain gradient descent algorithm in a general Hilbert space, which is an extension of a stochastic gradient descent algorithm for i.i.d.~samples in Hilbert spaces as discussed by \cite{MR2228737}.
Let $W$ be a Hilbert space with inner product $\langle~,~\rangle$. Consider a quadratic potential map $V:W\rightarrow \mathbb{R}$ given by 
$$V(w)=\frac{1}{2}\langle Aw,w \rangle+\langle B,w \rangle+C,$$
where $A:W\rightarrow W$ is a positive definite bounded linear operator with a bounded inverse i.e., $\|A^{-1}\|<\infty, B \in W~\text{and}~C\in \mathbb{R}.$ Then the gradient of \(V\) i.e., \(\nabla V: W \rightarrow W\) is 
by 
$$\nabla V(w)=Aw+B.$$
Note that for each single sample $z=(x,y)$,
$$\nabla V_z(w)=A(z)w+B(z),$$
where $A(z)$ is a random variable depending on $z$, given by the map $A:Z\rightarrow SL(W)$ taking values in $SL(W)$, the vector space of symmetric bounded linear operators on $W$ and $B:Z\rightarrow W$, where $B(z)$ is a $W$-valued random variable depending on $z$. Furthermore, \(V\) has a unique minimal point, \(w^\star\) such that \(\nabla V(w^\star)=0\) i.e., $w^\star=-A^{-1}B$. Hence \(w^*\) is a unique minimizer of \(V\) with zero mean i.e., \(\mathbb{E}_{z \sim \rho}[\nabla V_z(w^{\star})]=0\). We now aim to find an approximation of $w^{\star}$ and give an analysis on the sample complexity based on a Markov chain based gradient algorithm in Hilbert spaces.\par Let $(z_t)_{t\in \mathbb{N}}$ be a sequence of samples along a Markov chain trajectory defined on a state space $(Z,\mathcal{B}(Z))$ with a transition kernel $P$ and a unique stationary distribution $\rho$ and $(\gamma_t)_{t\in \mathbb{N}}$ a positive sequence of step sizes. Next, define an update formula for each $t\in\N$ by 
\begin{align}\label{generalupdt}
    w_{t+1}&=w_{t}-\gamma_t \nabla V_{z_t}(w_t)~\text{for some}~w_1\in W,
\end{align}
where $\nabla V_{z_t}: W \rightarrow W$ depends on the sample $z_t$ at time step \(t\). Note that \(w_{t+1}\) depends on the truncated history \((z_i)_{i=1}^t\). With slight abuse of notation, we shall write $\nabla V_z$ whenever the reference to a sample $z\in Z$ rather than to a time step $t$ in a sampling or update process seems to be more appropriate.  We additionally assume that \(\nabla V_z\) is Bochner integrable i.e., \(\int_Z \|\nabla V_z\|_W\,d\rho<\infty\), where \(\rho\) is a stationary  measure on $Z$. 
In order to approximate $w^{\star}$, we work under the following assumptions
\begin{assumption}\label{A1}
    Let $w^\star$ be a unique minimizer of $V$. We assume that there exists some $\sigma \geq 0$ such that for all $z \in Z$
  $$\|\nabla V_{z}(w^\star)\|^2 \leq \sigma^2.$$ 
\end{assumption} This assumption reflects the noise at optimum with mean 0. When $\sigma^2=0$, we obtain the deterministic algorithm.
\begin{assumption}\label{A2}
   For all $z\in Z$, the functions $V_{z}$ are $\eta-$smooth and $\kappa-$strongly convex such that we have
       \begin{equation}
       \nabla^2 V_{z}(w) \leq \eta~\text{and}~\nabla^2 V_{z}(w) \geq \kappa ~\text{for all}~w\in W.\label{consmo}
   \end{equation}
We will set $\alpha=\frac{\kappa}{\eta}$ where $\alpha \in (0,1]$ due to \eqref{consmo}.
\end{assumption}
 We now provide an upper bound on the distance between the sample-based model $w_t$ based on a stationary Markov chain, and the target model $w^\star$ in Hilbert spaces which results in Theorem \ref{thm:1} and Proposition \ref{rm:1}.
\begin{theorem} \label{thm:1}
Let \( W \) be a Hilbert space, and let \( V : W \to \mathbb{R} \) be a quadratic potential map with a unique minimizer \( w^\star \) satisfying Assumptions \ref{A1} and \ref{A2}. Consider a sequence of samples \( (z_t)_{t \in \mathbb{N}} \) generated by a time-homogeneous Markov chain on a measurable state space \( (Z, \mathcal{B}(Z)) \), starting from an arbitrary initial distribution. Denote by \( P \) the transition kernel of the chain and by \( \rho \) its stationary probability measure on \( Z \).
Let $\theta \in (\frac{1}{2},1)$ and consider $\gamma_t=\frac{1}{\eta t^\theta}$. Then for each $t\in\N$ and $w_t$ obtained by Eq.~\eqref{generalupdt},
$$\|w_{t}-w^{\star}\| \leq \mathcal{E}_\text{init}(t)+\mathcal{E}_\text{samp}(t)$$
where,
$$\mathcal{E}_\text{init}(t)\leq e^{\frac{2\alpha}{1-\theta}(1-t^{1-\theta})}\|w_{1}-w^{\star}\|;$$
and with probability at least $1-\delta$, with $\delta \in(0,1)$ in the space $Z^{t-1}$,
$$\mathcal{E}_\text{samp}(t) \leq  \;
\frac{5\,\sigma}{\eta}\,
\sqrt{
    \frac{C_{\theta}}{\delta}\,
   \alpha^{-\frac{\theta}{\,1 - \theta\,}}\,
}\,    \sqrt{\frac{t_{\mathrm{mix}}}{\,t^{\theta}\,}},$$
with $ C_{\theta}= 8 + \frac{2}{2\theta - 1} \left( \frac{\theta}{e(2 - 2^\theta)} \right)^{\theta / (1-\theta)}.$
    \end{theorem}
With similar arguments as in Theorem \ref{thm:1} we obtain an analogous result to Theorem \ref{thm:1} for the case $\theta=1$.
\begin{proposition} \label{rm:1}
    With all the assumptions as in Theorem \ref{thm:1}, but with $\theta=1$ and $\alpha\in \left(0,\frac{1}{2}\right)$, we obtain that for each $t\in\N$ and $w_t$ obtained by Eq.~\eqref{generalupdt},
    $$\|w_{t}-w^{\star}\|\leq  \mathcal{E}_\text{init}(t)+\mathcal{E}_\text{samp}(t)$$
    where,
$$\mathcal{E}_\text{init}(t)\leq \left(\frac{1}{t}\right)^\alpha\|w_{1}-w^{\star}\|;$$
and with probability at least $1-\delta,~\text{with}~\delta \in(0,1)$ in the space $Z^{t-1}$,
$$\mathcal{E}_\text{samp}(t)\;\le\; \frac{2\,\sigma}{\eta}\;\sqrt{\frac{7\,}{\delta\,\bigl(1 - 2\alpha\bigr)\,}}\,\sqrt{\frac{t_{\mathrm{mix}}}{\,t^{\alpha}\,}}.$$
\end{proposition}
Theorem \ref{thm:1} and Proposition \ref{rm:1} provide precise answers to question (ii) posed in the motivation by characterizing the convergence behavior of \( w_t \) toward \( w^\star \) based on a Markov chain trajectory. Specifically, for any \( \theta \in \left( \tfrac{1}{2}, 1 \right) \), the error satisfies  
\[
    \|w_t - w^\star\| = \mathcal{O}\!\left(t^{-\theta/2}\sqrt{t_{\mathrm{mix}}}\right),
\]
and in the boundary case \( \theta = 1 \), the rate becomes  
\[
    \|w_t - w^\star\| = \mathcal{O}\!\left(t^{-\alpha/2}\sqrt{t_{\mathrm{mix}}}\right),
\]
for some \( \alpha \in \left(0,\tfrac{1}{2}\right] \). When the Markov chain mixes rapidly, these bounds closely match the i.i.d.\ rates of \cite{MR2228737}, and are sharper in this sense.
\par
Now, before proceeding to the statement and proof of Theorem~\ref{thm:1} and Proposition~\ref{rm:1} i.e., in order to prove error bounds on the convergence of Markov chain gradient descent in Hilbert spaces, we establish some necessary preliminary results, namely Propositions~\ref{prop2} and~\ref{prop3}, which provide the groundwork for the proof of the Theorem~\ref{thm:1} and Proposition~\ref{rm:1}.
\par
Recall that throughout the subsequent proofs, we adopt a slight abuse of notation. In particular, we use $\nabla V_z$ to emphasize dependence on a sample $z \in Z$, and $\nabla V_t$ to emphasize dependence on the time step $t$, referring to $z_t$ when contextually appropriate.
\begin{proposition}\label{prop2}
Let \( W \) be a Hilbert space, and let \( V : W \to \mathbb{R} \) be a quadratic potential map with a unique minimizer \( w^\star \) satisfying Assumptions \ref{A1} and \ref{A2}. Consider a sequence of samples \( (z_i)_{i=1}^t \) with $t\in \mathbb{N}$ obtained along a time-homogeneous Markov chain on a measurable state space \( (Z, \mathcal{B}(Z)) \), starting from an arbitrary initial distribution. Denote by \( P \) the transition kernel of the chain and by \( \rho \) its stationary probability measure on \( Z \).
 Let $\theta \in \left(\frac{1}{2},1\right)$, then we have that
 \begin{equation}\label{eqinthm}
     \mathbb{E} \left[\sum_{i=1}^t \frac{1}{i^{2\theta}} \prod_{k=i+1}^t \left( 1 - \frac{\alpha}{k^\theta} \right)^2 \| \nabla V_{i}(\omega^*) \|^2\right] \leq \sigma^2 C_{\theta}\left( \frac{1}{\alpha} \right)^{\theta / (1-\theta)} \left( \frac{1}{t+1} \right)^\theta,
 \end{equation}
where $ C_{\theta}= \left(8 + \frac{2}{2\theta - 1} \left( \frac{\theta}{e(2 - 2^\theta)} \right)^{\theta / (1-\theta)}\right).$
\end{proposition}
\begin{proof}
The left-hand side of Eq.~\eqref{eqinthm}
\[
\mathbb{E} \left[\sum_{i=1}^t \frac{1}{i^{2\theta}} \prod_{k=i+1}^t \left( 1 - \frac{\alpha}{k^\theta} \right)^2
\|\nabla V_{i}(\omega^*) \|^2\right]
\]
can be expressed, due to the linearity of expectation, as
\[
\sum_{i=1}^t \frac{1}{i^{2\theta}} \prod_{k=i+1}^t \left( 1 - \frac{\alpha}{k^\theta} \right)^2
\biggl[
     \mathbb{E}\|\nabla V_{i}(w^\star)\|^2
\biggr],
\]
 which can be further simplified and upper bounded in the following way
\begin{align*}
   \sum_{i=1}^t \frac{1}{i^{2\theta}} &\prod_{k=i+1}^t \left( 1 - \frac{\alpha}{k^\theta} \right)^2\mathbb{E}\left[ \| \nabla V_{i}(\omega^*) \|^2\right]\\
&= \sum_{i=1}^t \frac{1}{i^{2\theta}} \prod_{k=i+1}^t \left( 1 - \frac{\alpha}{k^\theta} \right)^2\biggl(\int_Z \|\nabla V_{z}(w^\star)\|^2\rho_i(dz) \biggl) \\
&= \sum_{i=1}^t \frac{1}{i^{2\theta}} \prod_{k=i+1}^t \left( 1 - \frac{\alpha}{k^\theta} \right)^2\biggl(\int_Z \biggl(\int_Z \|\nabla V_{z}(w^\star)\|^2 P^{i-1}(z', dz) \biggr)\rho_1(dz') \biggl).
\end{align*}
Due to Assumption \ref{A1} we can further conclude that
\begin{align*}
\sum_{i=1}^t \frac{1}{i^{2\theta}} &\prod_{k=i+1}^t \left( 1 - \frac{\alpha}{k^\theta} \right)^2\mathbb{E}\left[ \| \nabla V_{i}(\omega^*) \|^2\right]\\
&\leq \sigma^2 \sum_{i=1}^t \frac{1}{i^{2\theta}} \prod_{k=i+1}^t \left( 1 - \frac{\alpha}{k^\theta} \right)^2   \\
&\leq \sigma^2 C_{\theta}\left( \frac{1}{\alpha} \right)^{\theta / (1-\theta)} \left( \frac{1}{t+1} \right)^\theta,
\end{align*}
where $ C_{\theta}= \left(8 + \frac{2}{2\theta - 1} \left( \frac{\theta}{e(2 - 2^\theta)} \right)^{\theta / (1-\theta)}\right).$\\
We have used Lemma \ref{smale2} (see \ref{appen}) to estimate the final expression.
 \end{proof}
Before presenting and proving Proposition \ref{prop3}, we first recall some established results on Markov chains that will be needed for the proof.
\begin{proposition}\label{TV1} Let $d(i) = \sup_{z \in Z} \| P^i(z,\cdot) - \rho(\cdot) \|_{\mathrm{TV}}$, where $P^i(z, \cdot)$ is the $i$-step transition probability of a time-homogenous Markov chain with stationary distribution $\rho$ on state space $Z$, and let $t_{\mathrm{mix}}$ denote the mixing time of the chain. Then, for all integer $i \geq 1$, we have $d(i) < 2^{-\frac{i}{t_{\mathrm{mix}}} + 1}$, and furthermore, for any integer $t \geq 1$, $\sum_{i=1}^t d(i) \leq 4 t_{\mathrm{mix}}$.
\end{proposition}
\begin{proof}
    We first observe that $d(i)$ is non-increasing function of $i$. When $l$ is a positive integer, then
\[d(l t_{\text{mix}}(\epsilon))\leq \bar{d}(t_{\text{mix}}(\epsilon))^l\leq (2\epsilon)^l,\]
where \(
\overline{d}(t) = \sup_{x,y \in X} \left\| P^t(x, \cdot) - P^t(y, \cdot) \right\|.\) For further details on the aforementioned inequality, see also \cite{MR3726904,MR4732980}.\\
In particular, taking $\epsilon=\frac{1}{4}$ above yields
\[d(l t_{\text{mix}})\leq 2^{-l}.\]
% Again, for all $x\in \mathbb{R}$,
% \begin{align*}
    % x-1<\lfloor x \rfloor&\leq x \leq \lceil x \rceil<x+1\\
% \text{i.e.,}~&x \geq \lfloor x \rfloor 
% \end{align*}
Let $l=\biggl \lfloor\frac{i}{t_{\text{mix}}} \biggl \rfloor$, then $\biggl \lfloor\frac{i}{t_{\text{mix}}} \biggl \rfloor\leq \frac{i}{t_{\text{mix}}}$. \\
Hence, since $d(i)$ is non-increasing, this implies that 
\[d(i)\leq d\biggl(\biggl \lfloor\frac{i}{t_{\text{mix}}} \biggl \rfloor t_{\text{mix}}\biggl)\leq 2^{-\biggl \lfloor\frac{i}{t_{\text{mix}}} \biggl \rfloor}<2^{-\frac{i}{t_{\text{mix}}}+1}\]
\[\implies d(i)<2^{-\frac{i}{t_{\text{mix}}}+1}.\]
Moreover,  $\sum\limits_{i=1}^td(i)\leq 2\sum\limits_{i=1}^t\frac{1}{2^{i/t_{\text{mix}}}}<\frac{1}{1-2^{-1/t_{\text{mix}}}}\leq 4t_{\text{mix}}$.\\
\end{proof}
\begin{proposition}\label{prop3}
Let \( W \) be a Hilbert space, and let \( V : W \to \mathbb{R} \) be a quadratic potential map with a unique minimizer \( w^\star \) satisfying Assumptions \ref{A1} and \ref{A2}. Consider a sequence of samples \( (z_i)_{i=1}^t \) with $t\in \mathbb{N}$ obtained along a time-homogeneous Markov chain on a measurable state space \( (Z, \mathcal{B}(Z)) \), starting from an arbitrary initial distribution. Denote by \( P \) the transition kernel of the chain and by \( \rho \) its stationary probability measure on \( Z \).
 Let $\theta \in \left(\frac{1}{2},1\right)$, then we have that
\begin{equation}\label{eqinthm1}
\begin{aligned}
\mathbb{E} \Bigg[
2 \sum_{1 \le i < j \le t} 
\frac{1}{i^{\theta}} \frac{1}{j^{\theta}}
\prod_{k=i+1}^{t} \left( 1 - \frac{\alpha}{k^{\theta}} \right)
&\prod_{l=j+1}^{t} \left( 1 - \frac{\alpha}{l^{\theta}} \right)
\left\langle \nabla V_i(\omega^*), \nabla V_j(\omega^*) \right\rangle
\Bigg]\\
&\le 16 \sigma^{2} t_{\mathrm{mix}} C_{\theta}
\left( \frac{1}{\alpha} \right)^{\frac{\theta}{1-\theta}}
\left( \frac{1}{t+1} \right)^{\theta}.
\end{aligned}
\end{equation}
where \(
C_{\theta} = 8 + \frac{2}{2\theta - 1} \left( \frac{\theta}{e(2 - 2^\theta)} \right)^{\theta / (1-\theta)}.
\)
\end{proposition}
\begin{proof}
 Let \((z_{i})_{i=1}^t\) with \(t \in \mathbb{N}\) be a sequence of samples obtained along a Markov chain trajectory with properties described as above.
Denote \[I_{i+1}^t=\prod_{k=i+1}^t \left( 1 - \frac{\alpha}{k^\theta} \right).\] 
Hence, the left-hand side of Eq.~\eqref{eqinthm1} can be expressed by
\[\mathbb{E} \Bigg[ 2 \sum_{1\leq i<j \leq t}  \frac{1}{i^{\theta}}\frac{1}{j^\theta} 
I_{i+1}^t I_{j+1}^t \left\langle \nabla V_i(\omega^*), \nabla V_j(\omega^*) \right\rangle \Bigg].\]
Then using the linearity of expectation, we obtain that
\begin{align*}
&\mathbb{E} \Bigg[ 2 \sum_{1\leq i<j \leq t}  \frac{1}{i^{\theta}}\frac{1}{j^\theta} 
I_{i+1}^t I_{j+1}^t \left\langle \nabla V_i(\omega^*), \nabla V_j(\omega^*) \right\rangle \Bigg] \\
&= 2 \sum_{1\leq i<j \leq t} \frac{1}{i^{\theta}}\frac{1}{j^\theta} I_{i+1}^t I_{j+1}^t \mathbb{E}\left[\left\langle \nabla V_i(\omega^*), \nabla V_j(\omega^*) \right\rangle \right],
\end{align*}
with
\begin{equation}\label{equality}
\mathbb{E}\left[\left\langle \nabla V_i(\omega^*), \nabla V_j(\omega^*) \right\rangle \right]
=\int_Z \left( \int_Z \langle \nabla V_i(w^*), \nabla V_j(w^*) \rangle P^{j-i}(z_i, dz_j)\right) \,\rho_i(dz_i).
\end{equation}
 We now use the fact that  if $T: B \to B'$ is a continuous linear operator between Banach spaces $B$ and $B'$, and $f: X \to B$ is Bochner integrable, then $Tf: X \to B'$ is Bochner integrable and hence the integration and the application of $T$ may be interchanged, i.e.,
\[
\int_E T f \, d\mu = T \int_E f \, d\mu,
\]
for all measurable subsets $E \in \Sigma$.
Now, since \(\nabla V_z\) is Bochner integrable,  \[\langle  \nabla V_i(w^\star), \int_Z \nabla V_j(w^\star)\,d\rho(z_j) \rangle=\int_Z \langle \nabla V_i(w^\star),  \nabla V_j(w^\star) \rangle \,d\rho(z_j)\] and since \(\int_Z \nabla V_j(w^\star)\,d\rho(z_j)=0\) i.e., \(\mathbb{E}_{z \sim \rho}[\nabla V_z(w^{\star})]=0\), we obtain that
\begin{align*}
\int_Z  \left\langle \nabla V_i(\omega^*), \nabla V_j(\omega^*) \right\rangle& 
P^{j-i}(z_i, dz_j)\\&= \int_Z \left\langle \nabla V_i(\omega^*), \nabla V_j(\omega^*) \right\rangle 
\left(P^{j-i}(z_i, dz_j)-\rho(\,dz_j)\right),
\end{align*}
Additionally, we may employ the following fact: if \(\mu\) is a finite signed measure on \((Z,\mathcal{B}(Z))\) and \(f:Z\to\mathbb{R}\) is any bounded measurable function, then
\[
\Bigl|\int_Z f\,d\mu\Bigr| \;\le\; \int_Z |f|\,d|\mu|,
\]
where \(|\mu|\) is the total variation measure of \(\mu\) (see Athreya and Lahiri \cite{MR2247694}).
\par
Hence, using the above two properties in Eq.~\eqref{equality} we obtain that
\begin{align*}
\mathbb{E}&\left[\left\langle \nabla V_i(\omega^*), \nabla V_j(\omega^*) \right\rangle \right] \\
&=   
\int_Z \left[ \int_Z \langle \nabla V_i(w^*), \nabla V_j(w^*) \rangle \left( P^{j-i}(z_i, dz_j)-  \, \rho(dz_j)\right) \right] \,\rho_i(dz_i) \\
    &\leq   
\int_Z \left[\left | \int_Z \langle \nabla V_i(w^*), \nabla V_j(w^*) \rangle \left( P^{j-i}(z_i, dz_j) 
    -  \, \rho(dz_j)\right) \right|\right] \,\rho_i(dz_i) \\
&\leq  
\int_Z\left[ \int_Z\left|\langle \nabla V_i(w^*), \nabla V_j(w^*) \rangle \right|
    \,d \left|P^{j-i}_{z_i} - \rho\right|(z_j) \right] \,\rho_i(dz_i)  \\
    &\leq  
 \sigma^2\int_Z  2\sup\{ P^{j-i}_{z_i}(A)-\rho(A):A\in \mathcal{B}(Z)\}\,\rho_i(dz_i)\\
&= 2\sigma^2  
 \int_Z  d_{\text{TV}} \big( P^{j-i}_{z_i}, \rho \big)  \,\rho_i(dz_i)\\
&\leq 2 \sigma^2  \, \sup_{z_i \in Z} \, 
    \Big[ d_{\text{TV}} \big( P^{j-i}_{z_i}, \rho \big) \Big] \int\rho_i(dz_i) \\
&= 2 \sigma^2 \, \sup_{z_i \in Z} \, 
    \Big[ d_{\text{TV}} \big( P^{j-i}_{z_i}, \rho) \Big] \\
&= 2\sigma^2~d(j-i),
\end{align*}
concluding that
\begin{align*}
&\mathbb{E} \Bigg[ 2 \sum_{1\leq i<j \leq t}  \frac{1}{i^{\theta}}\frac{1}{j^\theta} 
\prod_{k=i+1}^t \left( 1 - \frac{\alpha}{k^\theta} \right) 
\prod_{l=j+1}^t \left( 1 - \frac{\alpha}{l^\theta} \right)
\left\langle \nabla V_i(\omega^*), \nabla V_j(\omega^*) \right\rangle \Bigg]\\
&\leq 4 \sigma^2\sum_{1\leq i<j \leq t}\frac{1}{i^{\theta}}\frac{1}{j^\theta} I_{i+1}^t I_{j+1}^t d(j-i).
\end{align*}
Since for $\theta \in \left(\frac{1}{2},1\right)$, $I_{i+1}^t$ can be upper bounded according to Lemma \ref{smale1}, due to which we have
\begin{align*}
   I_{i+1}^tI_{j+1}^t& \leq e^{2 \alpha' ((i+1)^{1-\theta}-(t+1)^{1-\theta})}e^{2 \alpha' ((j+1)^{1-\theta}-(t+1)^{1-\theta})}\\
   &=e^{-4 \alpha' (t + 1)^{1-\theta}}e^{2 \alpha' (i+1)^{1-\theta}} e^{2 \alpha' (j+1)^{1-\theta}},
\end{align*}
with \(\alpha'=\frac{\alpha}{1-\theta}\) such that 
\begin{align*}
  4 \sigma^2&\sum_{1\leq i<j \leq t}\frac{1}{i^{\theta}}\frac{1}{j^\theta} I_{i+1}^t I_{j+1}^t d(j-i)\\& \leq 4 \sigma^2 e^{-4 \alpha' (t + 1)^{1-\theta}} 
\sum_{1\leq i<j \leq t} \frac{1}{i^{\theta}} \frac{1}{j^\theta}
e^{2 \alpha' (i+1)^{1-\theta}} e^{2 \alpha' (j+1)^{1-\theta}} d(j - i).
\end{align*}
Expanding, rearranging and finally using the fact that for $i < j$, $\frac{1}{j^\theta} < \frac{1}{i^\theta}$ gives
\begin{align*}
\sum_{1\leq i<j \leq t} \frac{1}{i^{\theta}} \frac{1}{j^\theta}
e^{2 \alpha' (i+1)^{1-\theta}} &e^{2 \alpha' (j+1)^{1-\theta}} d(j - i)\\
&\leq e^{2 \alpha' (t+1)^{1-\theta}} \sum_{i=1}^t \frac{1}{i^{2\theta}} e^{2 \alpha' (i+1)^{1-\theta}}\sum_{i=1}^td(i).
\end{align*}
Summarizing, we may conclude
\begin{align*}
\mathbb{E}& \Bigg[ 2 \sum_{1\leq i<j \leq t}  \frac{1}{i^{\theta}}\frac{1}{j^\theta} 
\prod_{k=i+1}^t \left( 1 - \frac{\alpha}{k^\theta} \right) 
\prod_{l=j+1}^t \left( 1 - \frac{\alpha}{l^\theta} \right) 
\left\langle \nabla V_i(\omega^*), \nabla V_j(\omega^*) \right\rangle \Bigg] \\
& \leq4 \sigma^2 e^{-4 \alpha' (t + 1)^{1-\theta}} 
\sum_{1\leq i<j \leq t} \frac{1}{i^{\theta}} \frac{1}{j^\theta}
e^{2 \alpha' (i+1)^{1-\theta}} e^{2 \alpha' (j+1)^{1-\theta}} d(j - i)\\
&\leq 4 \sigma^2 e^{-2 \alpha' (t+1)^{-\theta}}\sum_{i=1}^t \frac{1}{i^{2\theta}} e^{2 \alpha' (i+1)^{1-\theta}}\sum_{i=1}^td(i)\\
&\leq 16 \sigma^2 t_{\text{mix}}  C_{\theta} \left( \frac{1}{\alpha} \right)^{\theta / (1-\theta)} \left( \frac{1}{t+1} \right)^{\theta},
\end{align*}
where \(
C_{\theta} = 8 + \frac{2}{2\theta - 1} \left( \frac{\theta}{e(2 - 2^\theta)} \right)^{\theta / (1-\theta)}.
\)\\
Note that the last inequality follows from Proposition \ref{TV1} and Lemma \ref{smale2}.
\end{proof}
We are now finally in a position to give the proof of Theorem \ref{thm:1} and Proposition~\ref{rm:1}.
\subsection{Proof of Theorem \ref{thm:1}}
\begin{proof}
 Let \((z_{i})_{i=1}^t\) with \(t \in \mathbb{N}\) be a sequence of samples obtained along a Markov chain trajectory starting from an arbitrary initial distribution on the measurable state space $(Z, \mathcal{B}(Z))$ admitting a unique stationary distribution, \(\rho\) with transition kernel $P$. Let $t_{\mathrm{mix}}$ denote the mixing time of the chain. Let \(W\) be a general Hilbert space and recall the Markov chain gradient descent algorithm \begin{align*}
    w_{t+1}&=w_{t}-\gamma_t \nabla V_{z_t}(w_t)~\text{for some}~w_1\in W,
\end{align*}
where $V:W\rightarrow \mathbb{R}$ is a quadratic potential map given by 
$$V(w)=\frac{1}{2}\langle Aw,w \rangle+\langle B,w \rangle+C,$$
with its gradient given as $\nabla V(w)=Aw+B$ and \(w^\star\) its unique minimizer.
\par
Next, let us define a residual vector at time \( t \) as 
\[
r_{t} = w_t - w^\star,
\]
which is a random variable depending on \( (z_i)_{i=1}^{t-1} \in Z^{t-1}\) for \( t \geq 2 \). 
Since \( w_{t+1} = w_t - \gamma_t (A_t w_t + B_t) \), it follows that
\begin{align*}
r_{t+1} &= w_{t+1} - w^\star \\
&= w_t - \gamma_t (A_t w_t + B_t) +\gamma_tA_tw^\star-\gamma_tA_tw^\star-w^\star  \\
&= (I - \gamma_t A_t) r_t - \gamma_t (A_t w^\star + B_t).
\end{align*}
By recursion, we can express \( r_{t+1} \) as follows
\[
r_{t+1} = \prod_{k=1}^t (I - \gamma_k A_k) r_1 - \sum_{i=1}^t \gamma_i 
\left( \prod_{k=i+1}^t (I - \gamma_k A_k) \right) (A_i w^\star + B_i).
\]
To simplify the notation, we define a symmetric linear operator 
\[
X_{i+1}^t : W \to W,
\]
which depends on \( z_{i+1}, \ldots, z_t \) and is given by 
\[
X_{i+1}^t (z_{i+1}, \ldots, z_t) = \prod_{k=i+1}^t (I - \gamma_k A_k).
\]
Additionally, we set \( X_{i+1}^t = 1 \) if \( i \geq t \). We also define a vector \( Y_i \in W \), depending only on \( z_i \), as 
\[
Y_i = A_i w^\star + B_i=\nabla V_i(w^\star).
\]
Using these notations, \( r_{t+1} \) can be expressed as
\[
r_{t+1} = X_1^t r_1 - \sum_{i=1}^t \gamma_i X_{i+1}^t Y_i.
\]
In this expression, the first term \( X_1^t r_1 \) represents the accumulated error caused by the initial choice, while the second term \( \sum_{i=1}^t \gamma_i X_{i+1}^t Y_i \) has zero mean and represents the fluctuation caused by the random sample.
Based on this, we define the initial error as 
\[
\mathcal{E}_{\text{init}}(t+1) = \|X_1^t r_1\|,
\]
and the sample error as 
\[
\mathcal{E}_{\text{samp}}(t+1) = \left\| \sum_{i=1}^t \gamma_i X_{i+1}^t Y_i \right\|.
\]
By $\kappa I \leq A_k \leq \eta I$ which is due to Assumption \ref{A2},  and $\gamma_t = \frac{1}{\eta t^\theta}$ ($\theta \in \left(\frac{1}{2}, 1\right)$), then
\begin{align*}
\| X_{i+1}^t r_1 \| &\leq \prod_{k=i+1}^t \| I - \gamma_k A_k \| \| r_1 \| \\
&\leq \prod_{k=i+1}^t \left(1 - \frac{\alpha}{k^\theta}\right) \| r_1 \|, \quad \alpha = \frac{\kappa}{\eta}.
\end{align*}
Now applying Lemma \ref{smale1} for \(\theta \in \left(\frac{1}{2},1\right) \) to the above inequality, we obtain the result on the initial error.
We know that for any sum of vectors
\[
\left\| \sum a_i \right\|^2 = \left\langle \sum a_i, \sum a_i \right\rangle = \sum \|a_i\|^2 + 2 \sum_{i < j} \langle a_i, a_j \rangle.
\]
Applying this, we have
\begin{align*}
&\left\| \sum_{i=1}^t \gamma_i X_{i+1}^t Y_i \right\|^2 
= \sum_{i=1}^t \left\| \gamma_i X_{i+1}^t Y_i \right\|^2 
+ 2 \sum_{1\leq i<j \leq t} \langle \gamma_i X_{i+1}^t Y_i, \gamma_j X_{j+1}^t Y_j \rangle \\
&= \frac{1}{\eta^2}\sum_{i=1}^t \frac{1}{i^{2\theta}} \prod_{k=i+1}^t \left(1 - \frac{\alpha}{k^\theta} \right)^2 \|\nabla V_i(w^\star)\|^2 \\
&\quad +  \frac{2}{\eta^2}\sum_{1\leq i<j \leq t} \frac{1}{i^{\theta}} \frac{1}{j^\theta} 
\prod_{k=i+1}^t \left(1 - \frac{\alpha}{k^\theta} \right) 
\prod_{l=j+1}^t \left(1 - \frac{\alpha}{l^\theta} \right) 
\langle \nabla V_i(w^\star), \nabla V_j(w^\star) \rangle.
\end{align*}
Finally, applying Propositions \ref{prop2} and \ref{prop3}, to the above equality, we find that
\begin{align*}
\mathbb{E}\left[\left\|\sum_{i=1}^t \gamma_i X_{i+1}^t Y_i \right\|^2\right] 
&\leq \frac{\sigma^2}{\eta^2}  C_{\theta}\left( \frac{1}{\alpha} \right)^{\theta / (1-\theta)} \left(\frac{1}{t+1}\right)^\theta\\
&+ \frac{16\sigma^2}{\eta^2} C_{\theta}\left( \frac{1}{\alpha} \right)^{\theta / (1-\theta)}t_{\text{mix}} \left(\frac{1}{t+1}\right)^\theta,
\end{align*}
where $ C_{\theta}= 8 + \frac{2}{2\theta - 1} \left( \frac{\theta}{e(2 - 2^\theta)} \right)^{\theta / (1-\theta)}$.
Now using Markov's inequality, which states that for a non-negative random variable \( X \) and any \( \epsilon > 0 \)
\[
\mathbb{P}(X \geq \epsilon) \leq \frac{\mathbb{E}[X]}{\epsilon},
\]
we obtain for \( X = \mathcal{E}^2_{\text{samp}}(t) \) and \( t \geq 2 \) that
\begin{align*}
    \mathbb{P}(\mathcal{E}^2_{\text{samp}}(t)>\epsilon^2)&\leq \frac{\mathbb{E}[\mathcal{E}^2_{\text{samp}}(t)]}{\epsilon^2}\\
\text{which further implies that}~\mathbb{P}(\mathcal{E}^2_{\text{samp}}(t)\leq \epsilon^2)&\geq 1-\frac{\mathbb{E}[\mathcal{E}^2_{\text{samp}}(t)]}{\epsilon^2}.
\end{align*}
Taking $\delta=\frac{\mathbb{E}[\mathcal{E}^2_{\text{samp}}(t)]}{\epsilon^2}$, we get a probabilistic upper bound on the sample error i.e., with probability at least $1-\delta$, with $\delta \in(0,1)$ in the space $Z^{t-1}$,
$$\mathcal{E}^2_\text{samp}(t) \leq  \frac{\sigma^2}{\delta\eta^2}C_{\theta}\left( \frac{1}{\alpha} \right)^{\theta / (1-\theta)}\left( \frac{1}{t} \right)^{\theta}(1+16t_{\text{mix}}),$$
with $ C_{\theta}= 8 + \frac{2}{2\theta - 1} \left( \frac{\theta}{e(2 - 2^\theta)} \right)^{\theta / (1-\theta)}.$\\
Note that for $t=1$, $\mathcal{E}_\text{init}(t)=\|w_{1}-w^{\star}\|$ and $\mathcal{E}^2_\text{samp}(t)=0$ and after further simplification, we observe that $$\mathcal{E}_\text{samp}(t) \leq  \;
\frac{5\,\sigma}{\eta}\,
\sqrt{
    \frac{C_{\theta}}{\delta}\,
   \alpha^{-\frac{\theta}{\,1 - \theta\,}}\,
}\,    \sqrt{\frac{t_{\mathrm{mix}}}{\,t^{\theta}\,}},$$
with $\alpha \in (0,1]$. 
\end{proof}
Using arguments similar to those presented in the proof of Theorem~\ref{thm:1}, we provide a concise proof of Proposition~\ref{rm:1} below.
\subsection{Proof of Proposition~\ref{rm:1}}
\begin{proof}
Following the proof similar to that of Theorem \ref{thm:1}, we have observed that
\begin{align*}
&\left\| \sum_{i=1}^t \gamma_i X_{i+1}^t Y_i \right\|^2 
= \sum_{i=1}^t \left\| \gamma_i X_{i+1}^t Y_i \right\|^2 
+ 2 \sum_{1\leq i<j \leq t} \langle \gamma_i X_{i+1}^t Y_i, \gamma_j X_{j+1}^t Y_j \rangle \\
&= \frac{1}{\eta^2}\sum_{i=1}^t \frac{1}{i^{2\theta}} \prod_{k=i+1}^t \left(1 - \frac{\alpha}{k^\theta} \right)^2 \|\nabla V_i(w^\star)\|^2 \\
&\quad +  \frac{2}{\eta^2}\sum_{1\leq i<j \leq t} \frac{1}{i^{\theta}} \frac{1}{j^\theta} 
\prod_{k=i+1}^t \left(1 - \frac{\alpha}{k^\theta} \right) 
\prod_{l=j+1}^t \left(1 - \frac{\alpha}{l^\theta} \right) 
\langle \nabla V_i(w^\star), \nabla V_j(w^\star) \rangle.
\end{align*}
Now following arguments similar to that of Propositions \ref{prop2} and \ref{prop3}, and using Lemma \ref{smale2}  for \(\theta =1\) and \(\alpha\in\left(0,\frac{1}{2}\right)\) we observe that \(\mathbb{E}\left[\left\|\sum_{i=1}^t \gamma_i X_{i+1}^t Y_i \right\|^2\right]\) is bounded above by 
\begin{align*}  \frac{4\sigma^2}{\eta^2(1 - 2\alpha)} \left(\frac{1}{t+1}\right)^{2\alpha}+ \frac{2 \sigma^2}{\eta^2}\sum_{1\leq i<j \leq t} \frac{1}{i^{\theta}}\frac{1}{j^\theta} \prod_{k=i+1}^t \left( 1 - \frac{\alpha}{k}\right)  
\prod_{l=j+1}^t \left( 1 - \frac{\alpha}l \right)d(j-i).
\end{align*}
Observe that for \(\theta=1\), \(\sum_{1\leq i<j \leq t} \frac{1}{i}\frac{1}{j} \prod_{k=i+1}^t \left( 1 - \frac{\alpha}{k}\right)  
\prod_{l=j+1}^t \left( 1 - \frac{\alpha}{l} \right)d(j-i)\) is bounded above by 
\begin{align*}
\left(\frac{1}{t+1}\right)^{2\alpha}\sum_{1\leq i<j \leq t} \frac{1}{i}\frac{1}{j}(i+1)^\alpha(j+1)^\alpha d(j-i).
\end{align*}
 Expanding and rearranging the above inequality along with later using the fact that for $i < j$, $\frac{1}{j} < \frac{1}{i}$ we further obtain that 
\[\left(\frac{1}{t+1}\right)^{2\alpha}\sum_{1\leq i<j \leq t} \frac{1}{i}\frac{1}{j}(i+1)^\alpha(j+1)^\alpha d(j-i)\leq \sum_{i=1}^t\frac{1}{i^2}\left(\frac{i+1}{t+1}\right)^\alpha\sum_{i=1}^td(i).\]
Applying Lemma \ref{myineq}, we obtain that 
\begin{align*}
    \sum_{i=1}^t\frac{1}{i^2}\left(\frac{i+1}{t+1}\right)^\alpha\sum_{i=1}^td(i)\leq \frac{6}{1-\alpha}\left(\frac{1}{t+1}\right)^\alpha \left(4 t_{\text{mix}}\right).
\end{align*}
Hence, for \(\theta=1\) and \(\alpha\in\left(0,\frac{1}{2}\right)\) along with applying the Markov inequality for  \( X = \mathcal{E}^2_{\text{samp}}(t) \) we obtain that with probability at least $1-\delta$, with $\delta \in(0,1)$ in the space $Z^{t-1}$,
\begin{align*}
 \mathcal{E}^2_\text{samp}(t) &\leq  \frac{4\sigma^2}{\delta\eta^2}\left[\frac{1}{(1-2\alpha)}\left(\frac{1}{t}\right)^{2\alpha}+\frac{6 t_{\text{mix}}}{(1-\alpha)}\left(\frac{1}{t}\right)^{\alpha}\right],\\
 &\leq  \frac{4\sigma^2}{\delta\eta^2}\left(\frac{1}{1-2 \alpha}\right)\left(\frac{1}{t}\right)^{\alpha}\left(1+6t_{\text{mix}}\right).
\end{align*}
Hence after further simplifications we observe that $$\mathcal{E}_\text{samp}(t)\;\le\; \frac{2\,\sigma}{\eta}\;\sqrt{\frac{7\,}{\delta\,\bigl(1 - 2\alpha\bigr)\,}}\,\sqrt{\frac{t_{\mathrm{mix}}}{\,t^{\alpha}\,}}.$$
\end{proof}
\section{Application to regularized learning algorithm}\label{application}
Building on the discussions and results concerning the first two questions outlined in the motivation, we next focus on the third question, specifically, how can the results in the above sections be applied to a specific regularized learning method, namely the online regularized learning algorithm in reproducing kernel Hilbert spaces (RKHS)?  We first begin by recalling the structure of online regularized learning algorithm in RKHS for i.i.d.~samples.\par Let \( (z_t = (x_t, y_t))_{t \in \mathbb{N}} \) be a sequence of random samples independently distributed according to \( \rho \). We consider the online regularized learning algorithm
\begin{equation}\label{eq:ogd}
    f_{t+1} = f_t - \gamma_t \left( \left( f_t(x_t) - y_t \right) K_{x_t} + \lambda f_t \right), \quad f_1 \in \mathcal{H}_K, \\
    \text{with e.g., } f_1 = 0,
\end{equation}
where \( \lambda > 0 \) is the \textit{regularization parameter}, \( (\gamma_t)_{t \in \mathbb{N}} \) denotes the \textit{step-size sequence}, and \( (f_t)_{t \in \mathbb{N}} \) is the \textit{learning sequence}. Each iterate \( f_{t+1} \) depends on the history \( (z_i)_{i=1}^t \).

This update rule can be derived directly from the general \textit{gradient descent algorithm}
\begin{equation}\label{genupdt1}
    w_{t+1} = w_t - \gamma_t \nabla V_{z_t}(w_t), \quad w_1 \in W,
\end{equation}
by taking \( W = \mathcal{H}_K \) and, for a fixed \( z = (x,y) \in Z \), defining the quadratic potential
\begin{equation}\label{lossinhk}
    V_z(f) = \frac{1}{2} \left( f(x) - y \right)^2 + \lambda \| f \|_K^2.
\end{equation}
From \cite[Proposition~3.1]{MR2228737}, the gradient in \( \mathcal{H}_K \) is
\[
    \nabla V_z(f) = (f(x) - y)K_x + \lambda f.
\]
Substituting \( f = f_t \) and \( (x,y) = (x_t,y_t) \) yields
\[
    \nabla V_{z_t}(f_t) = (f_t(x_t) - y_t)K_{x_t} + \lambda f_t,
\]
and identifying \( w_t = f_t \) in Eq.~\eqref{genupdt1} recovers the online update in Eq.~\eqref{eq:ogd} in the RKHS setting.

In the standard learning framework on \(Z=X\times Y\) with the input space being \( X \), a compact metric space, the output space \( Y \subseteq \mathbb{R} \), and samples \( (z_t = (x_t, y_t))_{t \in \mathbb{N}} \)  drawn independently according to a probability measure \( \rho \) on \( Z = X \times Y \), the goal is to approximate the regression function
\[
    f_\rho(x) = \int_Y y \, d\rho(y \mid x),
\]
which uniquely minimizes the expected mean squared error
\[
    \mathcal{E}(f) = \int_{X \times Y} (f(x) - y)^2 \, d\rho(z).
\]
Note that \( K : X \times X \rightarrow \mathbb{R} \) is a Mercer kernel associated with RKHS \(\mathcal{H}_K\) and satisfies the reproducing property
\begin{equation}\label{reproducing_property}
    \langle K_x, g \rangle = g(x), \quad \forall g, \; \forall x, \; g \in \mathcal{H}_K.
\end{equation}
Furthermore, for a measure \(\mu\) on \(X\) and \(\lambda > 0\), the target function is
\begin{equation}\label{eq:flamstar}
    f_{\lambda, \mu} = \underset{f \in \mathcal{H}_K}{\mathrm{argmin}} \left\{ \int_X (f(x) - f_\rho(x))^2 \, d\mu + \lambda \| f \|_K^2 \right\}.
\end{equation}

Based on the above context, our focus now shifts to the Markov sampling scenario, where \((z_t)_{t\in \mathbb{N}}\) is a Markov chain with stationary distribution \(\rho\), where the chain does not necessarily start from the stationary distribution instead of i.i.d.~sampling. The connection between Eq.~\eqref{eq:ogd} and Eq.~\eqref{genupdt1} allows us to take advantage of the convergence analysis of the \textit{Markov chain gradient descent} (Theorem~\ref{thm:1} and Proposition~\ref{rm:1}) to infer results for online regularized setting based on a Markov chain by identifying \(w_{t+1} = f_{t+1}\) and \(w^* = f_{\lambda,\mu}\) in \(W = \mathcal{H}_K\). In this setting, our goal is to control the excess risk
\[
    \| f_{t+1} - f_{\rho} \|_\rho,
\]
which admits the standard decomposition 
\begin{equation}\label{decomp1}
    \| f_{t+1} - f_\rho \|_\rho
    \le \| f_{t+1} - f_{\lambda, \mu} \|_\rho
    + \| f_{\lambda, \mu} - f_\rho \|_\rho.
\end{equation}
Here, $\|f\|_\rho$ denotes the norm in $L^2_{\mu}(X)$, the Hilbert space of square-integrable functions with respect to the measure $\mu$. In this setting, $\mu$ represents the marginal probability measure on $X$ induced by the joint distribution $\rho$ on $X \times Y$.
We next derive upper bounds for each term on the right-hand side of Eq.~\eqref{decomp1}.
Before doing so, we briefly recall some relevant notions regarding compact operators on RKHS, which will be essential to prove subsequent results.
\begin{definition}
     For a  probability measure $\mu$ we define an integral operator $\opT{\mu}:L^2_{\mu}(X)\rightarrow L^2_{\mu}(X)$ as 
\begin{equation}\label{eq:Tkmu}
\opT{\mu}f(\cdot)=\int_X K(\cdot,x)f(x)d\mu(x),~~f\in  L^2_{\mu}(X),
\end{equation}
where $\opT{\mu}$ is a well-defined continuous and compact operator with $L^2_{\mu}(X)$ being the Hilbert space of square-integrable functions with respect to $\mu$. The \textit{regression function} $f_{\rho}$ is said to satisfy the \textbf{source condition} (of order $\beta$) if
$$f_\rho=\opT{\mu}^\beta(g)~\text{for some}~g\in L^{2}(\mu).$$ 
\end{definition}
Under suitable conditions on $K$, the operator $T_{K,\mu}$ is well-defined, continuous, and compact.
\par
In order to estimate the term $\|f_{\lambda, \mu}-f_{\rho}\|_{\rho}$, we recall Lemma 3 by \cite{MR2327597} i.e.,
\begin{proposition}[{\cite{MR2327597}}]
Suppose that $f_\rho$ satisfies the source condition (of order $\beta$) with $0<\beta \leq 1$. Then for some $g\in L^2_{\mu}(X)$, and for any $\lambda >0$, the following inequalities holds i.e., the upper bound between the target function $ f_{\lambda, \mu}$ and the regression function $f_{\rho}$ in the space  $L^2_{\mu}(X)$ is upper bounded as
 $$\|f_{\lambda, \mu}-f_\rho\|_{\rho}\leq \lambda^\beta \|g\|_{\rho}.$$
 \end{proposition}
 Next we focus on obtaining an upper bound on $\|f_{t+1}-f_{\lambda, \mu}\|_{\rho}$.\\
It can be shown and will be clarified later in the proof of Theorem \ref{thm:2} that the gradient \( \nabla V_z(f) \) satisfies, for all \( z \in Z \), the following inequalities
\[
\|\nabla V_z(f_{\lambda, \mu})\|^2 \leq \left( \frac{2M C_K^2(\lambda + C_K^2)}{\lambda} \right)^2
\]
and
\[
\lambda \leq \nabla^2 V_z(f) \leq \lambda + C_K^2,
\]
thus satisfying Assumptions \ref{A1} and \ref{A2} in Section \ref{sec:mcsg}, where \[C_K = \underset{x \in X}{\sup} \sqrt{K(x, x)},
\] and assume that \( C_K < \infty \). 
 Note that \(
\|f_{t+1} - f_{\lambda, \mu}\|_{\rho} = C_K \|f_{t+1} - f_{\lambda, \mu}\|_K.\)\par
We now consider a Markov chain \( (z_t)_{t \in \mathbb{N}} \), where \( z_t \in Z \subseteq X \times [-M, M] \) for all \( t \in \mathbb{N} \) with stationary distribution \( \rho \).
Given the learning rate \(
\gamma_t = \frac{1}{(\lambda + C_K^2)t^\theta},
\) with \( C_K\) defined as above, we involve Theorem \ref{thm:1} and Proposition \ref{rm:1} (see Section \ref{sec:mcsg} for results) to derive upper bounds on the error term \( \|f_t - f_{\lambda, \mu}\|_K \) for the cases \( \theta \in \left( \frac{1}{2}, 1 \right) \) and \( \theta = 1 \).

 \begin{theorem}\label{thm:2}
 Let $M > 0$ and define $Z = X \subset \mathbb{R}^n \times [-M, M]$. Consider a sequence of samples $(z_t)_{t\in\N}$ obtained along a Markov chain trajectory on the measurable space $(Z, \mathcal{B}(Z))$, and let $\rho$ denote its stationary joint probability measure on $Z$.
Let $\theta \in (\frac{1}{2},1)~\text{and}~\lambda>0$ and consider $\gamma_t= \frac{1}{(\lambda+C_K^2)t^\theta}$ and $\alpha=\frac{\lambda}{(\lambda +C_K^2)}\in (0,1]$. Then we have, for each $t\in\N$ and $f_t$ obtained by Eq.~\eqref{onlinealgo},
\begin{equation}\label{decomp}
    \|f_{t}-f_{\lambda, \mu}\|_{K}\leq \mathcal{E}_\text{init}(t)+\mathcal{E}_\text{samp}(t)
\end{equation}
where 
$$\mathcal{E}_\text{init}(t)\leq e^{\frac{2\alpha}{1-\theta}(1-t^{1-\theta})}\|f_{1}-f_{\lambda, \mu}\|_K;$$
and with probability at least $1-\delta$, with $\delta \in(0,1)$ in the space $Z^{t-1}$,
$$\mathcal{E}_\text{samp}(t) \leq  \frac{\sqrt{c'}}{ \lambda }\sqrt{
    \frac{C_{\theta}}{\delta}\,
   \alpha^{-\frac{\theta}{\,1 - \theta\,}}\,
}\,\sqrt{ \frac{t_{\mathrm{mix}}}{\,t^{\theta}\,}},$$
with $ C_{\theta}= 8 + \frac{2}{2\theta - 1} \left( \frac{\theta}{e(2 - 2^\theta)} \right)^{\theta / (1-\theta)}$ and $c'=4(MC_K^2)^2$.

\end{theorem}
\begin{proposition} \label{rm}
    With all the assumptions as in Theorem \ref{thm:2}, but with $\theta=1$ and $\lambda<C_K^2$ for which $\alpha=\frac{\lambda}{\lambda+C_K^2}\in \left(0,\frac{1}{2}\right)$, we obtain that for each $t\in\N$ and $f_t$ obtained by Eq.~\eqref{onlinealgo},
$$\|f_{t}-f_{\lambda,\mu}\|_{K}\leq \mathcal{E}_\text{init}(t)+\mathcal{E}_\text{samp}(t)$$
where,
$$\mathcal{E}_\text{init}(t)\leq \left(\frac{1}{t}\right)^\alpha\|f_{1}-f_{\lambda,\mu}\|_{K}; $$
and with probability at least $1-\delta,~\text{with}~\delta \in(0,1)$ in the space $Z^{t-1}$,
$$\mathcal{E}_\text{samp}(t)\;\le\; \frac{2\,\sigma}{\lambda}\;\sqrt{\frac{7\,}{\delta\,\bigl(1 - 2\alpha\bigr)\,}}\,\sqrt{\frac{t_{\mathrm{mix}}}{\,t^{\alpha}\,}}.$$
\end{proposition}
\begin{remark} 
Assuming without loss of generality that $0 < \lambda \leq 1$, then Theorem~\ref{thm:2} gives
\[
\| f_t - f_{\lambda,\mu} \|_K = \mathcal{O}\!\left(\sqrt{t_{\mathrm{mix}}}\,\lambda^{-\tau(\theta)}\, t^{-\theta/2}\right),
\qquad
\tau(\theta) = \frac{2 - \theta}{2(1 - \theta)},
\]
for $\theta \in \left(\tfrac{1}{2}, 1\right)$, where $\tau(\theta)$ increases from $\frac{3}{2}$ to $\infty$ as $\theta$ grows. When $\lambda$ is small, choosing $\theta$ close to $\tfrac{1}{2}$ yields the rate $\mathcal{O}\!\left(\sqrt{t_{\mathrm{mix}}}\,\lambda^{-3/2}\, t^{-1/4}\right)$, which limits the $\lambda$-penalty and thus mitigates the effect of slow mixing, at the cost of slower decay in $t$. Increasing $\theta$ accelerates the decay in $t$ but worsens the dependence on $\lambda$ (and hence on $t_{\mathrm{mix}}$ through the constant $\lambda^{-\tau(\theta)}$). Thus, in the presence of large $t_{\mathrm{mix}}$, smaller $\theta \in \left(\tfrac{1}{2}, 1\right)$ values help control the mixing-time penalty, whereas larger $\theta$ values are preferable when $\lambda\in(0,1]$ can be chosen moderately large.

\end{remark}

\begin{remark}
For the case $\theta = 1$ and $\lambda < C_K^2$, Proposition~\ref{rm} gives
\[
\|f_t - f_{\lambda,\mu}\|_K = \mathcal{O}\!\left(\sqrt{t_{\mathrm{mix}}}\,\lambda^{-1}\, \sqrt{\frac{1}{1-2 \alpha(\lambda)}}\, t^{-\frac{\alpha(\lambda)}{2}}\right),
\]
where $\alpha(\lambda) = \frac{\lambda}{2(\lambda + C_K^2)} \in \left(0, \frac{1}{2}\right)$ is a monotonically increasing function of $\lambda$ that attains its maximal value $1/2$ as $\lambda \to C_K^2$. Choosing $\lambda$ close to $C_K^2$ thus reduces the prefactor $\lambda^{-1}$ and increases the decay rate exponent $\alpha(\lambda)$, leading to a nearly optimal decay of $t^{-1/4}$ (up to the \(\sqrt{t_{\mathrm{mix}}}\) factor). However, this choice simultaneously increases the factor $\sqrt{\frac{1}{1-2 \alpha(\lambda)}}$, which diverges as $\lambda \to C_K^2$, making this choice suboptimal for the overall convergence rate. On the other hand, setting $\lambda$ very small amplifies the mixing-time penalty via a larger prefactor and a smaller $\alpha(\lambda)$, resulting in both a slower decay in $t$ and a stronger dependence on $t_{\mathrm{mix}}$. 
\end{remark}

We now give the proof of Theorem \ref{thm:2}.
\begin{proof}
    In order to apply Theorem \ref{thm:1} we first need to identify the equivalent conditions in RKHS for which we first proceed by identifying the target optimality condition, i.e., \(w^\star\) characterized by \(\mathbb{E}_{z\sim\rho} \left[\nabla V_z(w^\star)\right] = 0 \) as well as translation of Assumptions \ref{A1} and \ref{A2} after which we are in a position to apply Theorem \ref{thm:1}. It is important to note that, in identifying the optimality and fixed conditions, the expectations in the subsequent arguments are taken with respect to the stationary distribution 
\(\rho\) of the Markov chain.
    \par Let $W=\mathcal{H}_K$. Denote $J_x$ as the evaluation functional such that for any $x\in X \subset \mathbb{R}^d$, $J_x:\mathcal{H}_K\rightarrow \mathbb{R}$, with $J_x(f)=f(x)$ and for all $f\in\mathcal{H}_K$. Then $J_x^\star $ is denoted as the adjoint operator where $J_x^\star :\mathbb{R}\rightarrow \mathcal{H}_K$. Hence we obtain that 
$\langle J_x( f ), y \rangle _\mathbb{R}=\langle f(x), y \rangle _\mathbb{R}=yf(x)=y \langle f , K_x \rangle _{\mathcal{H}_K}=\langle f , y K_x \rangle _{\mathcal{H}_K} $. Also since $\langle J_x( f ), y \rangle _\mathbb{R}=\langle f, J_x^\star (y) \rangle _{\mathcal{H}_K}$, we obtain that $J_x^\star (y) = yK_x$.\par 
Define the linear operator $A_x:\mathcal{H}_K\rightarrow \mathcal{H}_K$ by $A_x = J_x^\star J_x + \lambda \opI$, where $\opI$ is the identity operator. Then 
$A_x(f)=J_x^\star J_x(f)+\lambda f=f(x)K_x+\lambda f$, whence $A_x$ is a random variable depending on $x$. Taking the expectation of $A_x$ , we have $\hat{A}=\mathbb{E}_x\left[A_x\right]=\opT{\mu} + \lambda \opI$, with $\opT{\mu}$ restricted on $\mathcal{H}_K$ and also since $J_x^\star J_x:\mathcal{H}_K\rightarrow \mathcal{H}_K$, where $J_x^\star J_x(f)=f(x)K_x$.\par
Additionally, define $B_z=J_x^\star(-y)=-yK_x\in \mathcal{H}_K$, where $B_z$ is a random variable depending on $z=(x,y)$. Taking expectation of $B_z$, we have $\hat{B}=\mathbb{E}_z\left[B_z\right]=\mathbb{E}_z\left[-yK_x\right] =\mathbb{E}_x\left[\mathbb{E}_y\left[-y\right]K_x \right]=\mathbb{E}_x\left[-f_\rho K_x \right]=-\opT{\mu} f_\rho$. 
Now, recall that $V:\mathcal{H}_K \rightarrow \mathbb{R}$, where $V$ is a quadratic potential  map whose general form is given by 
$$V(w)=\frac{1}{2}\langle Af,f \rangle+\langle B,f \rangle+C,$$
where $A:\mathcal{H}_K\rightarrow \mathcal{H}_K $ is a positive bounded linear operator with $\|A^{-1}\|<\infty$, $B \in \mathcal{H}_K$ and $C\in \mathbb{R}.$
Hence, $\nabla V: \mathcal{H}_K\rightarrow \mathcal{H}_K$ which is given by 
$$\nabla V(f)=Af+B.$$
Note that for each sample $z$,
$$\nabla V_z(f)=A(z)f+B(z)=A_xf+B_z,$$
where $A(z)$ is a random variable depending on $z=(x,y)$, given by the map $A:Z\rightarrow SL(\mathcal{H}_K)$ taking values $SL(\mathcal{H}_K)$, the vector space of symmetric bounded linear operators $\mathcal{H}_K$ and $B:Z\rightarrow \mathcal{H}_K$, where $B(z)$ is a $\mathcal{H}_K$ valued random variable depending on $z$.  Hence moving from a general setting to a specific case with $V_z(f)=\frac{1}{2}\left((f(x)-y)^2+\lambda \|f\|_K^2\right)$, along with $f_{\lambda, \mu}=(\opT{\mu}+\lambda \opI)^{-1}\opT{\mu}f_{\rho},$ we obtain
\begin{align*}
    \mathbb{E}_z\left[\nabla V_z(f_{\lambda, \mu})\right]&=\mathbb{E}_z\left[ A_z f_{\lambda,\mu}+B_z\right]\\
    &=\mathbb{E}_x\left[ A_x\right] f_{\lambda,\mu}+\mathbb{E}_z\left[B_z\right]\\
&=(\opT{\mu}+\lambda \opI)(\opT{\mu}+\lambda \opI)^{-1}\opT{\mu}f_{\rho}-\opT{\mu}f_{\rho}=0.\\
\end{align*}
Hence, in expectation $ f_{\lambda,\mu}$ is a minimizer of all $V_z$, for $z\in Z.$
Assumption \ref{A2} translates into,
$$\nabla^2V_z(f)\leq \eta~\text{and}~\nabla^2V_z(f)\geq \kappa \implies A_x\leq \eta~\text{and}~A_x\geq \kappa.$$
Since $\|\opT{\mu}\|\leq C_K^2$ (see e.g., \cite{MR2354721}), where $C_k=\underset{x\in X}\sup\sqrt{K(x,x)}$, and $\mathbb{E}_x\left[A_x\right]=\opT{\mu}+\lambda \opI,$ we obtain that $\kappa=\lambda$ and $\eta=\lambda+C_K^2.$ Moreover
\begin{align*}
    \|A_z f_{\lambda,\mu}+B_z\|&\leq \|A_x\|\| f_{\lambda,\mu}\|+\|B_z\|\\
&\leq (\lambda + C_K^2)\|(\opT{\mu}+\lambda \opI)^{-1}\opT{\mu}f_{\rho}\|+\|-yK_x\|\\
&\leq (\lambda + C_K^2)\|\hat{A}^{-1}\hat{B}\|+\|-y\|\|K_x\|\\
&\leq (\lambda + C_K^2)\frac{1}{\lambda}M C_K^2+M C_K^2\\
&\leq \frac{2M C_K^2(\lambda + C_K^2)}{\lambda}.
\end{align*}
Hence from the above inequality, we obtain that $ \|A_z f_{\lambda,\mu}+B_z\|^2\leq \left(\frac{2M C_K^2(\lambda + C_K^2)}{\lambda}\right)^2$ approving Assumption \ref{A1}.\par
We finally identify $f_t=w_t$, $f_{\lambda,\mu}=w^\star$, $\gamma_t= \frac{\lambda}{(\lambda+C_K^2)^2}\frac{1}{t^\theta}$ with $\theta \in (\frac{1}{2},1)$, $\sigma^2=\left(\frac{2M C_K^2(\lambda + C_K^2)}{\lambda}\right)^2$ and $W=\mathcal{H}_K$ in Theorem \ref{thm:1} and obtain an upper bound for the initial error, $\mathcal{E}_\text{init}(t)$, as $\mathcal{E}_\text{init}(t)\leq 2e^{\frac{2\alpha}{1-\theta}(1-t^{1-\theta})}\|f_{1}-f_{\lambda, \mu}\|_\rho$ and a probabilistic bound for the sampling error, $$\mathcal{E}_\text{samp}(t) \leq  \frac{\sqrt{c'}}{ \lambda }\sqrt{
    \frac{C_{\theta}}{\delta}\,
   \alpha^{-\frac{\theta}{\,1 - \theta\,}}\,
}\,\sqrt{ \frac{t_{\mathrm{mix}}}{\,t^{\theta}\,}},$$
with $ C_{\theta}= 8 + \frac{2}{2\theta - 1} \left( \frac{\theta}{e(2 - 2^\theta)} \right)^{\theta / (1-\theta)}$, $\alpha=\frac{\lambda}{(\lambda +C_K^2)}\in (0,1]$ and $c'=4(MC_K^2)^2$.
\end{proof}
\begin{remark}
Using similar arguments as above, the result of Proposition~\ref{rm} follows directly from a straightforward application of Corollary~\ref{rm:1}. However, for other values of $\alpha$, Lemma \ref{smale2} and Lemma \ref{myineq} in the appendix can also be used to determine the corresponding upper bounds. Specifically, when $\alpha = \frac{1}{2}$, the condition simplifies to $\lambda = C_K^2$. For $\alpha \in \left(\frac{1}{2}, 1\right)$, the condition becomes $\lambda > C_K^2$, and for $\alpha = 1$, it reduces to $C_K^2 = 0$.
\end{remark}
\section{Conclusion}
The goal for studying gradient descent algorithm for Markov chains is motivated by various examples of Markov chains in real world scenarios for examples time series, random walk on graphs, biological sequences etc. We specifically study the convergence analysis for Markov chain gradient descent algorithm in Hilbert spaces. This particular setting further helps us to extend the results to a particular example of gradient descent algorithm i.e., the online regularized learning algorithm in reproducing kernel Hilbert spaces for Markov samples for which we provide explicit convergence rates. The main results obtained in this paper extended the previously known result of \cite{MR2228737} of i.i.d. observations to the case of Markov chain samples.
\section*{Acknowledgements}
This research was carried out under the Austrian COMET program (project S3AI with FFG no. 872172, \url{www.S3AI.at}, at SCCH, \url{www.scch.at}),  which is funded by the Austrian ministries BMK, BMDW and the province of Upper Austria. The research reported in this paper has also been supported by the Federal Ministry for Climate Action, Environment, Energy, Mobility, Innovation and Technology (BMK), the Federal Ministry for Labour and Economy (BMAW), and the State of Upper Austria in the frame of the SCCH competence center INTEGRATE [(FFG grant no. 892418)] in the COMET - Competence Centers for Excellent Technologies Programme managed by Austrian Research Promotion Agency FFG.
\par
The author also gratefully acknowledges the financial support provided by the \textit{Iris Fischlmayr Stipendium} in JKU for carrying out this research.

\appendix
\renewcommand{\thetheorem}{A.\arabic{theorem}}
\renewcommand{\thelemma}{A.\arabic{lemma}}
\section{Some established inequalities}\label{appen}
\begin{lemma}[\cite{MR2228737}]\label{smale1}
For $\alpha\in(0,1]$ and $\theta \in [0,1)$, we observe that \[\prod_{k=i+1}^{t}\biggl(1-\frac{\alpha}{ k^\theta}\biggl)\leq e^{\biggl(\frac{2\alpha}{1-\theta}((i+1)^{1-\theta}-(t+1)^{1-\theta})\biggl)},\] and for $\theta=1$,
 $$\prod_{k=i+1}^{t}\biggl(1-\frac{\alpha}{ k^\theta}\biggl)\leq \biggl(\frac{i+1}{t+1}\biggl)^{\alpha}.$$
\end{lemma}

\begin{lemma}[\cite{MR2228737}]\label{smale2}
For specific choices of $\theta$ and $\alpha$, the function \[\psi_\theta(t, \alpha)=\sum_{i=1}^{t-1} \frac{1}{i^{2\theta}} \prod_{k=i+1}^{t-1} \left( 1 - \frac{\alpha}{k^\theta} \right)^2\] admits the following upper bounds.
Let \( \alpha \in (0, 1] \) and \( \theta \in \left( \frac{1}{2}, 1 \right) \). Then for \( t \in \mathbb{N} \),
\[
\psi_{\theta}(t + 1, \alpha) \leq e^{-2\alpha' (t+1)^{1 - \theta}} \sum_{k=1}^{t} \frac{1}{k^{2\theta}} e^{2\alpha' (k+1)^{1 - \theta}}\leq  C_\theta \left( \frac{1}{\alpha} \right)^{\theta / (1-\theta)} \left( \frac{1}{t+1} \right)^\theta,
\]
where
\[
C_\theta = 8 + \frac{2}{2\theta - 1} \left( \frac{\theta}{e(2 - 2^\theta)} \right)^{\theta / (1-\theta)}.
\]

If \( \theta = 1 \), and for \( \alpha \in (0, 1] \),
\[
\psi_1(t+1, \alpha) = \sum_{i=1}^t \frac{1}{i^2} \prod_{k=i+1}^t \left( 1 - \frac{\alpha}{k} \right)^2
\]
\[
\leq
\begin{cases} 
\frac{4}{1 - 2\alpha} (t+1)^{-2\alpha}, & \alpha \in (0, \frac{1}{2}), \\[8pt]
4(t+1)^{-1} \ln(t+1), & \alpha = \frac{1}{2}, \\[8pt]
\frac{6}{2\alpha - 1}(t+1)^{-1}, & \alpha \in (\frac{1}{2}, 1), \\[8pt]
6(t+1)^{-1}, & \alpha = 1.
\end{cases}
\]
\end{lemma}
\begin{lemma}[\cite{roy2025gradientdescentalgorithmHilbert}]\label{myineq}
    Let \(t \in \mathbb{N}\) and \(\alpha \in (0, 1]\). Then, the following bounds hold 
\[
\sum_{i=1}^t \frac{1}{i^2} \left(\frac{i+1}{t+1}\right)^{\alpha} \leq
\begin{cases}
\frac{6}{1-\alpha}(t+1)^{-\alpha}, & \text{if } \alpha \in (0,1), \\[8pt]
\frac{5 \ln (t+1)}{t+1}, & \text{if } \alpha = 1.
\end{cases}
\]
\end{lemma}

\end{document}